\documentclass[conference]{IEEEtran}
\IEEEoverridecommandlockouts

\usepackage{cite}
\usepackage{amsmath,amssymb,amsfonts}
\usepackage{algorithmic}
\usepackage{graphicx}
\usepackage{textcomp}
\usepackage{xcolor}
\usepackage{breqn}
\usepackage{amsthm}
\usepackage{algorithm}
\usepackage{mathtools}
\usepackage{multirow}
\usepackage{color}
\usepackage{subfigure}
\usepackage{placeins}
\usepackage{hyperref}
\usepackage{bm}
\usepackage{lipsum}
\usepackage{tikz}
\usetikzlibrary{positioning}
\usepackage{booktabs}        

\newtheorem{theorem}{Theorem}

\newtheorem{definition}{Definition}

\newcommand{\X}{{\mathcal{X}}}
\newcommand{\U}{{\mathcal{U}}}
\newcommand{\K}{{\mathcal{K}}}

\newcommand{\cS}{{\mathcal{S}}}

\def\BibTeX{{\rm B\kern-.05em{\sc i\kern-.025em b}\kern-.08em
	T\kern-.1667em\lower.7ex\hbox{E}\kern-.125emX}}

\begin{document}

\title{Robust Optimization Approach and Learning Based Hide-and-Seek Game for Resilient Network Design\\
	\thanks{A shorter version of this work has been accepted for publication in the \emph{IEEE International Conference on Communications (ICC) 2026} (manuscript \#1571207467). This manuscript substantially extends the conference submission with new models, solution methods, and numerical results.}
}

\author{
	\IEEEauthorblockN{Mohammad Khosravi and Setareh Maghsudi}
	\IEEEauthorblockA{Learning Technical Systems, Ruhr University of Bochum, Bochum, Germany}}

\maketitle


\begin{abstract}
		We study the design of resilient and reliable communication networks in which a signal can be transferred only up to a limited distance before its quality falls below an acceptable threshold. When excessive signal degradation occurs, regeneration is required through regenerators installed at selected network nodes. In this work, both network links and nodes are subject to uncertainty. The installation costs of regenerators are modeled using a budgeted uncertainty set. In addition, link lengths follow a dynamic budgeted uncertainty set introduced in this paper, where deviations may vary over time. Robust optimization seeks solutions whose performance is guaranteed under all scenarios represented by the underlying uncertainty set. Accordingly, the objective is to identify a minimum-cost subset of nodes for regenerator deployment that ensures full network connectivity, even under the worst possible realizations of uncertainty. To solve the problem, we first formulate it within a robust optimization framework, and then develop scalable solution methods based on column-and-constraint generation, Benders decomposition, and iterative robust optimization. In addition, we formulate a learning-based hide-and-seek game to further analyze the problem structure. The proposed approaches are evaluated against classical static budgeted robust models and deterministic worst-case formulations. Both theoretical analysis and computational results demonstrate the effectiveness and advantages of our methodology.
\end{abstract}



\section{Introduction}
\label{sec:introduction}

In recent years, network infrastructures have become one of the most fundamental concepts across various technology-related domains, ranging from international data transmission to supporting critical systems in modern societies. As these infrastructures have become increasingly integrated into various aspects of life, including social, economic, and technological aspects, the need to ensure their stability and performance has grown substantially. As a result, the concept of \textit{network resilience} has introduced to address this need. In fact, network resilience refers to the ability of a network to maintain its services, particularly in the presence of disruptions or adversarial attacks. In other words, resilient design ensures that networks can withstand unpredictable events, mitigate their immediate impacts, and adapt to continue essential long-term operations.

One of the key factors that must be considered in designing a resilient network is the physical limitation on how far a signal can travel without degrading in quality. Accordingly, this study focuses on resilient communication network design based on a threshold $d_{max}$ that represents the maximum distance a signal can traverse without loss of quality. To address this issue, signal regenerators are deployed at selected network nodes to restore and maintain signal integrity.

Since regenerators are relatively expensive devices (see \cite{mertzios2012placing}), an important objective in resilient network design is to minimize the number of regenerators installed while still ensuring the network remains fully operational. This gives rise to the classical \textit{Regenerator Location Problem} (RLP). In its seminal form, introduced in \cite{sen2008sparse}, various approaches are presented for identifying the smallest subset of nodes at which regenerators should be placed, given a maximum possible distance $d_{max}$, to ensure connectivity among all nodes. In this sense, establishing communication between nodes separated by a distance more than $d_{max}$ is only possible by passing through nodes equipped with regenerators.

Based on previous studies (see \cite{borella1997optical,mukherjee2000wdm,zymolka1999examination}), the main causes of signal degradation can be summarized to transmission impairments such as attenuation, dispersion, and cross-talk. Inspired by \cite{chen2010regenerator}, there are three common methods for signal regeneration: \textit{1R}, which amplifies signals to restore their strength; \textit{2R}, which involves both amplification and reshaping to correct distortions; and \textit{3R}, which combines amplification, reshaping, and retiming to fully recover the signal’s quality and timing.

In addition to this physical constraint, real-world problems often involve varying installation and maintenance costs for regenerators due to geographic location, responsible service providers, and seasonal weather conditions. To account for these factors, a weighted version of the RLP was introduced in \cite{chen2010regenerator}, in which the objective alters from minimizing the number of regenerators to minimizing their overall installation and maintenance cost.

In this paper, we focus on the \textit{Robust Weighted} version of the RLP, where both edge lengths and node costs are defined within a \textit{budgeted uncertainty set} (refer to \cite{bertsimas2003robust}). We begin by summarizing all sources of transmission impairments along the network edges. Accordingly, each edge is assumed to have a nominal length corresponding to its physical distance, along with a nominal deviation range that reflects the maximum possible utilization of that edge during signal transmission, which in turn leads to the highest degree of signal degradation. This factor may vary over time due to dispersion and cross-talk caused by the traffic traversing the edge at any given time. Thus, we define a \textit{dynamic budgeted uncertainty} set in which the length deviation of each link is time-dependent and bounded by its nominal maximum.

Similarly, we assume that the installation and maintenance costs of regenerators at network nodes also arise from a budgeted uncertainty set. In both contexts, the adversary is given a budget that allows it to activate deviations for a subset of nodes or edges. The robust weighted RLP can therefore be formulated as follows.

\begin{definition}
	Consider a communication network in which both edge lengths and node costs are uncertain and given through a budgeted uncertainty set. Let $d_{\max}$ denotes the maximum possible signal transmission distance. The objective of the robust weighted RLP is to identify a subset of network nodes with minimum total cost, under adversarial attack, on which regenerators are to be placed. In addition, the placement must guarantee the existence of a valid path between every pair of nodes such that no segment of any path without an internal regenerator exceeds the maximum distance $d_{\max}$.
\end{definition}


\subsection{Application Areas}
\label{subsec:application}
\subsubsection*{Optical Networks}
\label{subsec:optical-networks}
Optical networks are among the most fundamental components of global communication infrastructure, enabling the simultaneous transmission of large volumes of data with high bandwidth and low latency. The feasible transmission distance of optical fiber links depends on environmental and technological factors, and the installation and maintenance costs of optical regenerators vary with geographic location. The robust weighted RLP accounts for these uncertainties through budgeted uncertainty sets for both node costs and edge lengths, while ensuring that regenerator placement maintains network connectivity and remains reliable and cost-effective under fluctuating conditions (see \cite{chen2010regenerator,sen2008sparse,yildiz2015regenerator}).

\subsubsection*{Satellite Communication Networks}
\label{subsec:satellite-communication-networks}
Satellite communication systems serve as a primary means of connecting remote and underserved regions by establishing long-distance wireless links between ground stations and orbiting satellites. The deployment and maintenance costs of regenerative equipment at ground stations can be highly uncertain due to varying infrastructure requirements, while atmospheric effects introduce variability in effective transmission distances. The robust weighted RLP provides a systematic approach for flexible placement of regenerative nodes, ensuring that signal quality is preserved across dynamic and uncertain satellite channels (see \cite{ma2022uncertainty,kodheli2020satellite,jiang2022robust}).

\subsubsection*{Emergency and Public Safety Networks}
\label{subsec:emergency-and-public-safety-networks}
Emergency response and public safety networks must be designed to maintain functionality under adverse and uncertain conditions, such as natural disasters or infrastructure failures. In such settings, both the costs associated with deploying mobile or temporary regeneration units and the availability of feasible routes and links are uncertain and may deviate from their nominal values. Applying robust weighted RLP models enables decision makers to strategically position regenerative devices and ensure that critical communication paths remain operational even under worst-case scenarios (see \cite{wan2017smart,vamvakas2019risk,casoni2019emergency}).


\subsection{Related Works}
\label{subsec:relatedworks}
The classical version of the RLP has been extensively studied in the literature. To begin with, it has been shown that the RLP is NP-complete \cite{chen2010regenerator,sen2008sparse}. In addition, \cite{chen2010regenerator} formulated the RLP as a Steiner arborescence problem with a unit-degree constraint at the root node and developed a branch-and-cut algorithm to solve it. A polyhedral investigation of the RLP is presented in \cite{li2017regenerator}, while a flexible variant is examined in \cite{yildiz2017regenerator}. The generalized regenerator location problem, introduced in \cite{chen2015generalized}, focuses on connecting only a selected subset of nodes by installing regenerators on another specific subset.

Several extensions of the RLP have also been investigated. In \cite{yildiz2015regenerator}, resilient network design is studied under different survivability conditions, including resilience against regenerator failures and node failures. Their results indicate that the latter case is computationally more demanding. A fault-tolerant version of the RLP is proposed in \cite{rahman2014optimal}, where the authors introduce an integer programming formulation containing an exponential number of set-covering inequalities and develop a branch-and-cut approach to solve it efficiently. Another study, \cite{li2020branch}, proposes a flow-based model that accounts for all possible link failure scenarios and solves it using a branch-and-Benders-cut algorithm inspired by Benders decomposition \cite{benders2005partitioning}. Their computational analysis demonstrates that the latter method offers significantly better solution quality, though both methods face scalability challenges on large instances.

Research focusing specifically on the robust version of the RLP remains relatively limited. In \cite{yan2018robust}, a greenfield scenario with time-varying traffic demands is examined, and a regenerator assignment strategy based on the probability of needing a regenerator at each node is proposed. Robust optimization approaches have been applied to related network design problems, typically considering uncertainty in link capacities, costs, or traffic demands (see \cite{bertsimas2011theory,montemanni2005benders,alves2015robust,jackiewicz2025recoverable,pessoa2015robust}). In our recent work \cite{khosravi2026robust}, we focused on the robust weighted variant of the fault-tolerant RLP, where a subset of network links may fail and regenerator installation costs arise from a discrete uncertainty set. We first proved that this problem is NP-complete, and then proposed both flow-based and cut-based integer programming formulations. We further introduced an iterative solution framework suitable for large-scale networks. Our results show that this approach outperforms the method from \cite{li2020branch}.


\subsection{Our Contribution}
\label{subsec:contribution}
In summary, the contributions of this paper are as follows.
\begin{itemize}
	\item We introduce a robust version of the weighted RLP in which node installation costs are modeled using a static budgeted uncertainty set. Furthermore, we define a new dynamic budgeted uncertainty set to model edge lengths.
	\item We formulate the robust placement problem as a mixed-integer programming model under both static and dynamic budgeted uncertainty sets, ensuring feasibility under the corresponding worst-case scenarios.
	\item We develop three alternative solution methods for the dynamic formulation that are suitable for large-scale instances. These methods include column-and-constraint generation, Benders decomposition, and iterative robust optimization approaches.
	\item We propose a novel learning-based hide-and-seek (HSL) framework that frames the robust optimization problem as a repeated game between a learner and an adversary.
	\item We conduct a numerical study to evaluate the performance of our approaches in comparison to the static robust formulations and the deterministic worst-case scenario, demonstrating that our methods provide cost-efficient and resilient placement solutions.
\end{itemize}


\subsection{Organization}
\label{subsec:organization}
The remainder of this paper is organized as follows. In Section~\ref{sec:problem-statement}, we formally define the robust weighted regenerator location problem and introduce the corresponding uncertainty sets. We also establish the computational complexity of the problem. Section~\ref{sec:mathematical-formulations} presents the core modeling framework. We first introduce a graph transformation that is necessary to obtain a well-defined weighted RLP. We then describe an equivalent problem formulation that forms the basis of the mathematical models developed in this study. In addition, a preprocessing procedure aimed at reducing the computational effort is introduced prior to presenting the integer programming formulations. This section concludes by detailing the mathematical formulations of the weighted RLP under the deterministic worst-case scenario, as well as under static and dynamic robust optimization settings. In Section~\ref{sec:alternative-solution-approaches}, we develop alternative solution methods for the dynamic robust weighted RLP, including column-and-constraint generation, Benders decomposition, and iterative robust optimization approaches. Section~\ref{sec:learning-based-hide-and-seek-game} introduces the proposed learning-based hide-and-seek framework to address the robust problem. Section~\ref{sec:experiments} reports the computational results and provides a comparative analysis of all methods proposed in this paper. Finally, Section~\ref{sec:conclusion} concludes the paper and outlines potential directions for future research. For completeness, Appendix~\ref{app:proof-theorem-complexity} presents the formal proof of the problem’s computational complexity. Appendix~\ref{app:performance-profile} provides a detailed description of the performance profile methodology used in the numerical evaluation.




\section{Problem Statement}
\label{sec:problem-statement}

Consider a network represented by a graph $G=(V,E,D)$, where $V$ is the set of vertices of size $n$, $E$ is the set of edges of size $m$, and $D$ is a matrix specifying the edge lengths. Let $d_{max}$ denote the maximum possible distance for signal transmission. It is straightforward to verify that any edge whose length exceeds $d_{max}$ would lead to signal transmission failure and can therefore be removed from the network. Hence, we assume throughout this paper, that all edges have length less than or equal to $d_{max}$.

As previously noted, the objective of the classical RLP is to identify a subset of nodes of minimum cardinality at which regenerators are to be placed. In the weighted RLP, the objective turns into minimizing the total installation cost, since real-world deployment and maintenance costs vary across nodes. In this paper, we focus on the robust weighted RLP, in which the node installation costs and edge lengths are subject to uncertainty.

Robust optimization problems require an uncertainty set $\U$ describing possible scenarios. The structure of this set plays a key role in determining both the theoretical and practical difficulty of the problem. One of the most commonly used uncertainty sets is the budgeted uncertainty set. Two variants are typically considered, discrete and continuous budgeted sets \cite{goerigk2024benchmarking}. In the discrete variant, the adversary’s decision variables are binary, whereas in the continuous version they are continuous. In this paper, we focus on the discrete budgeted uncertainty set and, for simplicity, refer to it as the budgeted uncertainty set.

We assume that the cost of installing regenerators at the network nodes is defined through a budgeted uncertainty set of the form:

\footnotesize
\[
\mathcal{U}_v = \Big\{ \bm{\tilde{c}}:\ \tilde{c}_v = \underline{\tilde{c}}_v + \tilde{d}_v z_v,\ \sum_{v} z_v \le \Gamma_v,\ z_v \in \{0,1\} \Big\}
\]
\normalsize
where each node $v \in V$ has a nominal installation cost $\underline{\tilde{c}}_v$, a maximum deviation $\tilde{d}_v$, and the adversary may increase the costs of at most $\Gamma_v$ nodes to their upper bound $\underline{\tilde{c}}_v + \tilde{d}_v$. Since this uncertainty set does not vary over time, it represents a static version; therefore, we refer to it as the static budgeted uncertainty set.

To bring the model closer to real-world applications, we introduce a new budgeted uncertainty structure for edge lengths, which we call the dynamic budgeted uncertainty set. This set is time-dependent and may vary periodically. It can be formulated as follows:

\footnotesize
\[
\mathcal{U}^t_e = \Big\{ \bm{c}^t:\ c_e^t = \underline{c}_e + d_e^t y_e^t,\ d_e^t \in[0,d_e], \ \sum_{e} y_e^t \le \Gamma_e,\ y_e^t \in \{0,1\} \Big\}\\
\]
\normalsize
where each edge $e\in E$ is characterized by a nominal length $\underline{c}_e$ and a time-dependent deviation $d_e^t$. At each time $t$, a realization $d_e^t\in[0,d_e]$ is drawn, and (similar to the static case) the adversary may select up to $\Gamma_e$ edges to raise to their upper bound $\underline{c}_e + d_e^t$.

To address the robust counterpart of the RLP under these budgeted uncertainty sets, we adopt the classical Min--Max criterion. The goal is to identify a minimum-cost solution while accounting for the fact that the adversary chooses the scenarios that maximizes our objective. This leads to the following formulation:

\footnotesize
\begin{align*}
\min_{\pmb{x}\in\X} \ \max_{\pmb{\tilde{c}}\in\U_v} \ & \ \pmb{\tilde{c}}^{\intercal} \pmb{x}\\
\text{s.t.} \quad & \ \X\sim\U_e
\end{align*}
\normalsize

Note that $\U_v$ directly influences the objective function, whereas $\U_e^t$ forms the network structure and therefore affects the constraints of the problem. Since these uncertainty sets change the decision making process, it is essential to understand the computational complexity of the problem. Consequently, before introducing solution methods for the robust weighted RLP under the given budgeted uncertainty sets, we first establish its computational complexity.

\begin{theorem}\label{theorem:complexity}
	The robust regenerator location problem under the budgeted uncertainty set is NP-hard.
\end{theorem}
\begin{proof}
	See Appendix~\ref{app:proof-theorem-complexity}.
\end{proof}


\section{Mathematical Formulations}
\label{sec:mathematical-formulations}

In this section, we begin by presenting a graph transformation that simplifies the mathematical formulation of the problem. This transformation, inspired by the approach in \cite{sen2008sparse}, is a key component for implementing our solution methods. We then introduce an equivalent problem to the robust weighted RLP, which serves as the foundation for deriving the mathematical models. Next, we describe a preprocessing procedure that must be applied as the initial step of the solution process, as it significantly reduces computation time.

We subsequently develop three mathematical formulations for the weighted RLP. We first derive the model for the deterministic weighted RLP, in which both edge lengths and node costs are fixed at their respective upper bounds. We then extend the methods to two robust settings. The first adopts a classical robust optimization perspective, assuming that both edge lengths and node costs are defined by static budgeted uncertainty sets. The second, which represents the main focus of this paper, considers edge lengths given by a dynamic budgeted uncertainty set while node costs remain subject to static budgeted uncertainty.


\subsection{Graph Transformation}
\label{sec:graph-transformation}

Given a graph $G$ and the maximum possible transmission distance $d_{max}$, we begin by running an all-pairs shortest path algorithm on $G$. An edge $(i,j)$ is retained only if the shortest-path distance between nodes $i$ and $j$ does not exceed $d_{max}$; the collection of all such edges forms the set $E_M$. The resulting graph $M=(V,E_M)$ is referred to as the \textit{transformed graph}. If $M$ is complete, then every pair of nodes in the RLP instance can communicate without the need for regenerators. Otherwise, any pair of nodes not connected by an edge in $M$ (referred to as NDC pairs) can communicate only through a path whose internal nodes are equipped with regenerators. Consequently, the RLP can be equivalently reformulated on the transformed graph $M$ as follows:
\begin{definition}[RLP on M]
	Determine the smallest subset of nodes at which regenerators are installed such that every NDC pair in $M$ can communicate via a path whose internal nodes all contain regenerators.
\end{definition}
In instances where edge lengths are uncertain, solving robust shortest-path problems is required to determine the shortest distances between every pair of nodes. For this reason, each of our formulations begins by adapting the integer programming model for the shortest path problem (SHP) from \cite{montemanni2005benders} to account for uncertainty. This allows us to construct the transformed graph $M$. We then proceed analogously by modifying the flow-based RLP formulation introduced in \cite{khosravi2026robust}.


\subsection{Equivalent Problem}
\label{sec:equivalent-problem}

In this section, in a similar way used in \cite{khosravi2026robust}, we introduce problems that are equivalent to the RLP and serve as the foundation for the solution methods developed in this work.
\begin{definition}[Dominating Set]
	Given an undirected graph $G^{\prime}=$($V^{\prime},E^{\prime}$), a subset $T\subseteq V^{\prime}$ is a dominating set if every vertex $v^{\prime}\in V^{\prime}\backslash T$ is adjacent to at least one vertex in $T$.
\end{definition}
A \textit{connected dominating} set is a dominating set $T$ for which the induced subgraph $G^{\prime}_T=$($T$,$E^{\prime}_T$) is connected, where $E^{\prime}_T$ denotes the set of edges in $E^{\prime}$ with both endpoints in $T$. The problem of identifying a connected dominating set of minimum cardinality is known as the \textit{minimum connected dominating set problem} (MCDSP).

Prior work \cite{chen2010regenerator,lucena2010reformulations,gendron2014benders} has established that the MCDSP and the RLP are equivalent on the transformed graph $M$, using the \textit{maximum leaf spanning tree problem} as an intermediate step in the reduction. This equivalence does not hold on the original graph $G$, but it is central to our modeling on $M$. Motivated by this connection, an equivalent definition of the RLP (expressed in terms of connected dominating sets and used to derive our mathematical models) can be stated as follows:
\begin{definition}[Equivalent RLP on $M$]
	Determine a subset of nodes of minimum cardinality for regenerator placement such that the subgraph induced by this subset is connected and every node outside the subset is adjacent to at least one node within it.
\end{definition}


\subsection{Preprocessing}
\label{subsec:Preprocessing}

Before proceeding to the formulations, we describe a preprocessing step to practically reduce the computational time. To tackle the problem effectively, we begin by applying a heuristic method to reduce the computational time needed. This reduction is crucial to simplify subsequent steps of the solution process. A key observation is that any node on $M$ with a degree of one requires the placement of regenerators on its neighbor, because such a node must rely on its only neighbor for communication with the rest of the network. Consequently, the first step involves identifying all nodes in the network with a degree of one. This can be calculated in a polynomial time with $O$($n^2$). Once these nodes are identified, their neighboring nodes are added to the set of solution $\mathcal{L}^{0}$, which represents the initial placement of regenerators. Following this initial placement, we can start solving our models with a given partial solution referred to as the warm-start. This technique will reduce the solution time by reducing the branch-and-bound tree size for the out-of-the-box solvers like CPLEX.

\begin{algorithm}
	\caption{Preprocessing Algorithm}
	\label{alg:preprocessing}
	\begin{algorithmic}[1]
		\STATE \textbf{Input:} Network graph $G=(V,E)$
		\vspace{3pt}
		\STATE Initialize $\mathcal{L}^{0} \leftarrow \emptyset$.
		\STATE Compute the degree $\deg(v)$ of each node $v \in V$.
		\FOR{each node $v \in V$}
		\IF{$\deg(v) = 1$}
		\STATE Let $u$ be the unique neighbor of $v$.
		\STATE Add $u$ to the initial placement set: $\mathcal{L}^{0} \leftarrow \mathcal{L}^{0} \cup \{u\}$.
		\ENDIF
		\ENDFOR
		\STATE Mark $\mathcal{L}^{0}$ as the warm-start placement set.
		\STATE \textbf{Output:} Initial regenerator placement $\mathcal{L}^{0}$.
	\end{algorithmic}
\end{algorithm}


\subsection{Deterministic Worst-Case}
\label{subsec:deterministic-worst-case}

In this setting, we address the nominal versions of both the SHP and the RLP. That is, the problems are treated as fully deterministic, with no uncertain parameters. Thus, all edge lengths and node costs are fixed at their upper bounds: node costs take the value $\underline{\tilde{c}}_v + \tilde{d}_v$ for all $v \in V$, and edge lengths are set to $\underline{c}_e + d_e$ for all $e\in E$. Using these values, we first formulate the SHP to construct the underlying network.

Given a source node $s \in V$ and a destination node $t \in V$, the classical shortest path problem determines a path of minimum total. Let $x_e \in \{0,1\}$ denote a binary decision variable with $x_e = 1$ if edge $e$ is selected for the path from $s$ to $t$, and $x_e = 0$ otherwise. The resulting formulation of the shortest path problem is as follows:

\footnotesize
\begin{align*}
\min \ & \ \sum_{e \in E} \ (\underline{c}_v + d_v) \ x_e\\
\text{s.t.} \ & \ \sum_{e \in \delta^+_v} x_e - \sum_{e \in \delta^-_v} x_e =
\begin{cases}
1 & \text{if } v = s, \\
-1 & \text{if } v = t, \\
0 & \text{otherwise},
\end{cases}
\quad & \forall v \in V,\\
& \ x_e \in \{0,1\} & \forall e \in E
\end{align*}
\normalsize
where, $\delta^+_v$ and $\delta^-_v$ denote the out-degree and in-degree of node $v$, respectively. Since all edge lengths are non-negative, we can use this formulation in our setting. Once the transformed graph $M$ is obtained, we proceed by formulating the flow-based IP to solve the RLP. To this end, we introduce a binary decision variable $\tilde{x}_v \in \{0,1\}$, where $\tilde{x}_v = 1$ indicates that node $v$ is selected for regenerator installation, and $\tilde{x}_v = 0$ otherwise.

\footnotesize 
\begin{align}
\min \ & \ \sum_{v\in V} (\underline{\tilde{c}}_v + \tilde{d}_v) \tilde{x}_v \label{dwc-1}\\
\text{s.t.} \ & \ \sum_{e\in \delta^+_p} f^{e}_{pq} - \sum_{e\in \delta^-_p} f^{e}_{pq} = \tilde{x}_p \tilde{x}_q & \forall p,q\in V \label{dwc-2}\\
& \ \sum_{e\in \delta^+_q} f^{e}_{pq} - \sum_{e\in \delta^-_q} f^{e}_{pq} = -\tilde{x}_p \tilde{x}_q & \forall p,q\in V  \label{dwc-3}\\
& \ \sum_{e\in \delta^+_v} f^{e}_{pq} - \sum_{e\in \delta^-_v} f^{e}_{pq} = 0 & \forall p,q\in V,\; \forall v\in V\backslash\{p,q\} \label{dwc-4}\\
& \ \sum_{i\in\mathcal{N}_v} \tilde{x}_i \ge 1, & \forall v \in V  \label{dwc-5}\\
& \ f^{e}_{pq} \le \min \{\tilde{x}_i,\tilde{x}_j\} & \forall p,q\in V,\; \forall e=(i,j)\in E \label{dwc-7}\\
& \ f^{e}_{pq} \le \min \{\tilde{x}_p,\tilde{x}_q\} & \forall p,q\in V,\; \forall e\in E \label{dwc-8}\\
& \ f^{e}_{pq} \in [0,1], & \forall p,q\in V,\; \forall e\in E \label{dwc-9}\\
& \ \tilde{x}_v \in \{0,1\}, & \forall v \in V \label{dwc-10}
\end{align}
\normalsize
where, $\mathcal{N}_v$ denotes the set of neighboring nodes of $v$ in $M$. In this formulation, constraints (\ref{dwc-2},\ref{dwc-3},\ref{dwc-4}) ensure the existence of at least one path connecting every pair of nodes selected for regenerator placement. Furthermore, constraint (\ref{dwc-5}) guarantees that each node in the network is adjacent to at least one regenerator-equipped node.

Using constraints (\ref{dwc-2},\ref{dwc-3},\ref{dwc-4}), a unit of flow is sent from every node $p \in V$ to every node $q \in V$ with $p \neq q$, through some path $f$. In this construction, the total outgoing flow at node $p$ exceeds its incoming flow by one unit whenever both $p$ and $q$ belong to the regenerator-equipped set; otherwise, the inflow and outflow at $p$ must be equal. A symmetric condition applies at node $q$: its incoming flow exceeds its outgoing flow by one unit if both nodes are selected, and the two flows are equal otherwise. For any internal node on a feasible $p-q$ path, the incoming and outgoing flows must be equal.

This formulation introduces nonlinearity in constraints (\ref{dwc-2}) and (\ref{dwc-3}). To address this, we apply a standard McCormick linearization by introducing an auxiliary variable, as follows:

\footnotesize 
\begin{align*}
& t_{pq} = \tilde{x}_p \tilde{x}_q & \forall p,q\in V\\
& t_{pq} \le \min(\tilde{x}_p,\tilde{x}_q) & \forall p,q\in V\\
& t_{pq} \ge \max(0,\tilde{x}_p - (1-\tilde{x}_q))& \forall p,q\in V
\end{align*}
\normalsize
thus, $t_{pq}=1$ if and only if both $\tilde{x}_p=1$ and $\tilde{x}_q=1$. In this sense, we can obtain the linear formulation of the deterministic worst-case (DWC) model for the RLP as follows:

\footnotesize
\begin{align}
\min \ & \ \sum_{v\in V} (\underline{\tilde{c}}_v + \tilde{d}_v) \tilde{x}_v \\
\text{s.t.} \ & \ \sum_{e\in \delta^+_p} f^{e}_{pq} - \sum_{e\in \delta^-_p} f^{e}_{pq} = t_{pq} & \forall p,q\in V \\
& \ \sum_{e\in \delta^+_q} f^{e}_{pq} - \sum_{e\in \delta^-_q} f^{e}_{pq} = -t_{pq} & \forall p,q\in V \\
& \ \sum_{e\in \delta^+_v} f^{e}_{pq} - \sum_{e\in \delta^-_v} f^{e}_{pq} = 0 & \forall p,q\in V,\; \forall v\in V\backslash\{p,q\} \\
& \ \sum_{i\in\mathcal{N}_v} \tilde{x}_i \ge 1, & \forall v \in V \\
& \ t_{pq} \le \min(\tilde{x}_p,\tilde{x}_q), & \forall p,q\in V \\
& \ t_{pq} \ge \max(0,\tilde{x}_p - (1-\tilde{x}_q)), & \forall p,q\in V \\
& \ f^{e}_{pq} \le \min \{\tilde{x}_i,\tilde{x}_j\} & \forall p,q\in V,\; \forall e=(i,j)\in E \\
& \ f^{e}_{pq} \le \min \{\tilde{x}_p,\tilde{x}_q\} & \forall p,q\in V,\; \forall e\in E \\
& \ f^{e}_{pq} \in [0,1], & \forall p,q\in V,\; \forall e\in E \\
& \ \tilde{x}_v \in \{0,1\}, & \forall v \in V \\
& \ t_{pq} \in \{0,1\}, & \forall p,q\in V 
\end{align}
\normalsize


\subsection{Static Budgeted Uncertainty}
\label{subsec:static-budgeted-uncertainty}

Given a source node $s \in V$ and a destination node $t \in V$, in the first uncertain setting we seek a path whose cost is minimized under the corresponding worst-case realization, as defined by the following budgeted uncertainty set:

\footnotesize
\[
\mathcal{U}_e^{\prime} = \Big\{ \bm{c}:\ c_e = \underline{c}_e + d_e y_e,\ \sum_{e} y_e \le \Gamma_e,\ y_e \in \{0,1\} \Big\}.\\
\]
\normalsize

Now, let $x_e \in \{0,1\}$ be a similar binary decision variable as introduced in section \ref{subsec:deterministic-worst-case}. The robust shortest path problem can be formulated as:

\footnotesize
\begin{equation*}
\min_{x \in \mathcal{X}} \left\{ \sum_{e \in E} c_e x_e + \max_{z \in \mathcal{Z}} \sum_{e \in E} d_e x_e z_e \right\},
\end{equation*}
\normalsize
where $\mathcal{X}$ denotes the set of feasible $s-t$ paths, enforced through flow conservation constraints:

\footnotesize
\begin{equation*}
\sum_{e \in \delta^+(v)} x_e - \sum_{e \in \delta^-(v)} x_e =
\begin{cases}
 1 & \text{if } v = s, \\
-1 & \text{if } v = t, \\
 0 & \text{otherwise},
\end{cases}
\quad \forall v \in V.
\end{equation*}
\normalsize

The inner maximization (adversarial) problem over $z$ is linear subject to the budgeted uncertainty constraint. By dualizing this problem, we obtain a tractable robust counterpart to construct the graph $M$.

\footnotesize
\begin{align*}
\min \quad & \sum_{e \in E} c_e x_e + \Gamma \pi + \sum_{e \in E} \lambda_e, \\
\text{s.t.} \quad & \pi + \lambda_e \geq d_e x_e \quad & \forall e \in E, \\
& x_e \in \{0, 1\} \quad & \forall e \in E, \\
& \lambda_e \geq 0 & \forall e \in E, \\
& \pi \geq 0
\end{align*}
\normalsize

Once $M$ has been determined, we can formulate the mathematical model for solving the robust Min--Max RLP under the static budgeted uncertainty set, and for simplicity we refer to it as the static RLP. For this purpose, we first express the adversarial problem corresponding to a given solution $\tilde{x}$ (defined similar to that in Section~\ref{subsec:deterministic-worst-case}), as follows:

\footnotesize
\begin{align*}
\max & \sum_{i\in V} \tilde{c}_i \tilde{x}_i \tilde{z}_i \tag{ADV}\\
\text{s.t.} & \sum_{v\in V} \tilde{z}_i \leq \Gamma_v \\
& \tilde{z}_v \in [0,1] & \forall v\in V
\end{align*}
\normalsize
Now we can substitute the inner maximization of the Min--Max formulation with the dual of the ADV. Applying the same linearization technique used in Section~\ref{subsec:deterministic-worst-case}, the formulation of the robust static budgeted (RSB) problem becomes:

\footnotesize
\begin{align}
\min \ & \ \sum_{v\in V} \tilde{c}_v \tilde{x}_v + \Gamma_v\pi + \sum_{v\in V} \lambda_i \label{SBRLP-1}\\
\text{s.t.} \ & \ \sum_{e\in \delta^+_p} f^{e}_{pq} - \sum_{e\in \delta^-_p} f^{e}_{pq} = t_{pq} & \forall p,q\in V \label{SBRLP-2}\\
& \ \sum_{e\in \delta^+_q} f^{e}_{pq} - \sum_{e\in \delta^-_q} f^{e}_{pq} = -t_{pq} & \forall p,q\in V \label{SBRLP-3}\\
& \ \sum_{e\in \delta^+_v} f^{e}_{pq} - \sum_{e\in \delta^-_v} f^{e}_{pq} = 0 & \forall p,q\in V,\; \forall v\in V\backslash\{p,q\} \label{SBRLP-4}\\
& \ \sum_{i\in\mathcal{N}_v} \tilde{x}_i \ge 1, & \forall v \in V \label{SBRLP-5}\\
& \ \pi + \lambda_v \geq \tilde{c}_v \tilde{x}_v & \forall v\in V \label{SBRLP-6}\\
& \ t_{pq} \le \min(\tilde{x}_p,\tilde{x}_q), & \forall p,q\in V \label{SBRLP-7}\\
& \ t_{pq} \ge \max(0,\tilde{x}_p - (1-\tilde{x}_q)), & \forall p,q\in V \label{SBRLP-8}\\
& \ f^{e}_{pq} \le \min \{\tilde{x}_i,\tilde{x}_j\} & \forall p,q\in V,\; \forall e=(i,j)\in E \label{SBRLP-9}\\
& \ f^{e}_{pq} \le \min \{\tilde{x}_p,\tilde{x}_q\} & \forall p,q\in V,\; \forall e\in E \label{SBRLP-10}\\
& \ f^{e}_{pq} \in [0,1], & \forall p,q\in V,\; \forall e\in E \label{SBRLP-11}\\
& \ \tilde{x}_v \in \{0,1\}, & \forall v \in V \label{SBRLP-12}\\
& \ t_{pq} \in \{0,1\}, & \forall p,q\in V \label{SBRLP-13}\\
& \ \lambda_v \ge 0, & \forall v \in V \label{SBRLP-14}\\
& \ \pi \ge 0 \label{SBRLP-15}
\end{align}
\normalsize

\subsection{Dynamic Budgeted Uncertainty}
\label{subsec:dynamic-budgeted-uncertainty}

In practical optical networks, physical factors such as temperature variations, and maintenance activities introduce time-dependent fluctuations in effective signal reach and fiber quality. To account for these effects, we adopt the dynamic budgeted uncertainty set for edge lengths. Accordingly, we first solve a robust shortest path problem under dynamic budgeted uncertainty (DSP) to capture the periodic evolution of effective link lengths over the planning horizon. The resulting graph reflects these fluctuations. We then solve the static robust weighted RLP on this graph to obtain a stable and implementable regenerator placement plan.

Let $\mathcal{T}$ denote a finite planning horizon divided into discrete time periods $t \in \mathcal{T}$. At each time $t$, every edge $e \in E$ attains a length realization $c_e^t$ derived from $\mathcal{U}_e^t$. The DSP between each pair of nodes is thus formulated as follows:

\footnotesize
\begin{align}
	\min_{\pmb{x}} \ & \max_{t} \ \sum_{e\in E} c_e^t x_e^t \label{eq:dsp}\\
	\text{s.t.} \ & \sum_{e\in \delta^+_v} x_e^t - \sum_{e\in\delta^-_v} x_e^t =
	\begin{cases}
		1, & i=p,\\
		-1, & i=q,\\
		0, & \text{otherwise,}
	\end{cases} & \ \forall v\in V,\, t\in\mathcal{T}, \\
	& x_e^t \in \{0,1\}, & \forall e\in E,\, t\in\mathcal{T}
\end{align}
\normalsize

The objective determines the path with minimum cost under the worst-performing time period. By applying standard dualization to the inner maximization over $\mathcal{U}_e^t$, the mathematical formulation of DSP, used to generate the communication graph $M$, can be written as follows:

\footnotesize
\begin{align}
	\min \ & \sum_{e\in E} \underline{c}_e x_e^t + \Gamma_e \pi^t + \sum_{e\in E} \lambda_e^t \label{eq:dsp-rob}\\
	\text{s.t.} \ & \pi^t + \lambda_e^t \ge d_e^t x_e^t & \forall e\in E,\, t\in\mathcal{T} \\
	&\sum_{e\in \delta^+_v} x_e^t - \sum_{e\in\delta^-_v} x_e^t =
	\begin{cases}
	1, & i=p,\\
	-1, & i=q,\\
	0, & \text{otherwise,}
	\end{cases} & \ \forall v\in V,\, t\in\mathcal{T}, \\
	& \pi^t \ge 0 & \forall t\in \mathcal{T}\\
	& \lambda_e^t \ge 0 & \forall e\in E, t\in \mathcal{T}\\
	& x_e^t \in \{0,1\} & \forall e\in E, t\in\mathcal{T}
\end{align}
\normalsize

We can now proceed to solve the static version of the robust RLP using the IP formulation \ref{SBRLP-1}--\ref{SBRLP-15}, as presented in Section~\ref{subsec:static-budgeted-uncertainty}. We refer to this method as RDB.


\section{Alternative Solution Approaches}
\label{sec:alternative-solution-approaches}

The solution approach introduced in Section~\ref{subsec:dynamic-budgeted-uncertainty} addresses the robust weighted RLP under both static and dynamic budgeted uncertainty sets. It first constructs the time-dependent graph $M$ using the robust shortest path formulation under dynamic budgeted uncertainty set. It subsequently solves the robust weighted RLP on $M$ under the static budgeted uncertainty set. Although effective, this approach can become computationally demanding for large-scale instances, motivating the development of more scalable techniques. In this section, we introduce three methods designed to enhance the efficiency of the RDB approach for such instances.

\subsection{Column-and-Constraint Generation}
\label{subsec:column-and-constraint-generation}

Since the dynamic parameters are already captured in forming the graph $M$, the remaining task is to solve the robust weighted RLP under the static budgeted uncertainty. To this end, we apply a \textit{column-and-constraint generation} (CCG) algorithm tailored to the robust weighted RLP under the static budgeted uncertainty on $M$ to effectively solve the problem.

Given a selected set of candidate worst-case scenarios $\mathcal{S}$ derived from the uncertainty sets $\mathcal{U}_v$ and the obtained $\mathcal{U}_e$ that define $M$, the \textit{Master Problem} (MP) solves:

\footnotesize
\begin{align*}
\min_{\tilde{\pmb{x}}\in\mathcal{X}} \ \max_{\pmb{s}\in\mathcal{S}} \ & \ \tilde{\pmb{c}}^{\pmb{s}} \tilde{\pmb{x}} \tag{MP}\\
\text{s.t.} \quad & \ \X\sim\U_e
\end{align*}
\normalsize
Then, for the optimal placement $\tilde{\pmb{x}}^*$ found from the MP, the adversary solves the following subproblem to identify a new worst-case realization.

\footnotesize
\begin{align*}
\max_{\tilde{\pmb{c}}\in\mathcal{U}_v} \ & \ \tilde{\pmb{c}} \ \tilde{\pmb{x}}^* \tag{SP}\\
\text{s.t.} \quad & \ \pmb{x}^*\sim\X\sim\U_e
\end{align*}
\normalsize

If the adversarial objective exceeds the objective value of MP, the corresponding scenario is added to $\mathcal{S}$, and the MP is solved again. Although the CCG algorithm operates solely on the static RLP, it takes as input the graph $M$ generated in the dynamic stage and thus implicitly incorporates time-dependent variations at every iteration. The procedure terminates once the duality gap between the lower bound of MP  and the upper bound of SP falls below a given threshold $\varepsilon$. This process is summarized in Algorithm~\ref{alg:ccg}.
\begin{algorithm}
	\caption{Column-and-Constraint Generation Algorithm}
	\label{alg:ccg}
	\begin{algorithmic}[1]
		\STATE \textbf{Input:} Time-dependent data $\{d_e^t\}$, $\Gamma_e$, $\mathcal{U}_v$, $\mathcal{U}_{e}$.
		\STATE Compute communication graph $M$.
		\STATE Initialize scenario set $\mathcal{S} \leftarrow \emptyset$, given threshold $\varepsilon$.
		\REPEAT
		\STATE Solve the \textbf{MP} on $M$ to obtain $\tilde{\pmb{x}}^*$.
		\STATE Solve the \textbf{SP} for $\tilde{\pmb{x}}^*$.
		\IF{$\text{objective}_{SP} > \text{objective}_{MP}$}
		\STATE Add corresponding scenario to $\mathcal{S}$.
		\ENDIF
		\UNTIL{($\text{objective}_{SP} - \text{objective}_{MP}) < \varepsilon$}
		\STATE \textbf{Output:} Robust regenerator placement $\tilde{\pmb{x}}^*$.
	\end{algorithmic}
\end{algorithm}

\subsection{Benders Decomposition}
\label{subsec:benders-decomposition}

In the \textit{Benders Decomposition} (BDC) framework, we follow a structure similar to the CCG algorithm. Therefore, the DSP first provides the input data for the static RLP. In this case, the Benders framework further decomposes the static RLP into a master placement problem (MP) and a network-feasibility subproblem (SP). The MP determines a candidate regenerator placement $\tilde{\pmb{x}}$, through the following model:

\footnotesize
\begin{align*}
	\min \ & \ \sum_{v\in V} \underline{\tilde{c}}_v \tilde{x}_v + \Gamma_v \pi + \sum_{v\in V} \lambda_v + \theta \tag{MP}\\
	\text{s.t.} \ & \ \pi + \lambda_v \ge \tilde{d}_v \tilde{x}_v & \forall v\in V\\
	& \tilde{x}_v \in \{0,1\} & \forall v\in V
\end{align*}
\normalsize
Here, $\theta$ represents the subproblem’s value, which captures the connectivity feasibility cost over the graph $M$. Given $\tilde{\pmb{x}}$, the SP checks whether the chosen regenerator placement maintains connectivity across $M$ using the robust edge lengths $\pmb{c}^*$ for all $p$--$q$ pair of nodes:

\footnotesize
\begin{align*}
	\min \ & \ \sum_{e\in E_M} c^*_e f_{pq}^{e} \tag{SP}\\
	\text{s.t.} \ & \ \sum_{e\in \delta^+_v} f_{pq}^e - \sum_{e\in \delta^-_v} f_{pq}^e =
	\begin{cases}
	1, & \text{if } i = p,\\
	-1, & \text{if } i = q,\\
	0, & \text{otherwise,}
	\end{cases}
	& \ \forall v \in V, (p,q) \in \mathcal{P}\\
	& \ f_{pq}^{e} \le \min\{\tilde{x}_i,\tilde{x}_j\} & \forall (i,j)\in E_M\\
	& \ f_{pq}^{e} \in \{0,1\} & \forall e\in E_M
\end{align*}
\normalsize
where $\mathcal{P}$ is the set of all source-destination pairs that must be connected.

\begin{algorithm}[H]
	\caption{Benders Decomposition Algorithm}
	\label{alg:benders}
	\begin{algorithmic}[1]
		\STATE \textbf{Input:} Time-dependent data $\{d_e^t\}$, $\Gamma_e$, $\mathcal{U}_v$, $\mathcal{U}_{e}$.
		\STATE Compute communication graph $M$.
		\STATE Initialize set of Benders cuts $\mathcal{C} \leftarrow \emptyset$.
		\REPEAT
		\STATE Solve the \textbf{MP} with cuts $\mathcal{C}$ to obtain $\tilde{x}^*$.
		\STATE Solve the \textbf{SP} for $\tilde{x}^*$.
		\IF{SP is infeasible}
		\STATE Generate \textit{feasibility cut} and add to $\mathcal{C}$.
		\ELSE
		\STATE Generate \textit{optimality cut} $\theta \ge \sum_{(p,q)\in E_M} c_e^* {f_{pq}^{e}}^*$ and add to $\mathcal{C}$.
		\ENDIF
		\UNTIL{no violated cuts remain}
		\STATE \textbf{Output:} Optimal regenerator placement $\tilde{x}^*$.
	\end{algorithmic}
\end{algorithm}

\subsection{Iterative Robust Optimization}
\label{subsec:iterative-robust-optimization}

The \textit{Iterative Robust Optimization} (IRO) approach can also be adapted to solve the robust weighted RLP under both static and dynamic budgeted uncertainty set. Here, the DSP and static RLP stages are solved alternately, refining both the effective graph $M$ and the regenerator placement $\tilde{\pmb{x}}$ at each iteration.

At iteration $k$, the following sequence of steps is performed. First, we solve the DSP to update robust lengths and obtain $M^{(k)}$. In the next step, we solve the static RLP on $M^{(k)}$ to obtain $\tilde{\pmb{x}}^{(k)}$. Then we evaluate the worst-case realization $(\tilde{c}^{(k)}, c^{(k)})$. This process iteratively continues until the improvement is less than or equal to a given threshold $\varepsilon$. We can capture this method into the Algorithm~\ref{alg:iro}.

\begin{algorithm}[H]
	\caption{Iterative Robust Optimization Algorithm}
	\label{alg:iro}
	\begin{algorithmic}[1]
		\STATE \textbf{Input:} Initial parameters $(\underline{c}, \underline{\tilde{c}})$, $\Gamma_e$, $\Gamma_v$.
		\STATE Initialize iteration counter $k \leftarrow 0$, given threshold $\varepsilon$.
		\REPEAT
		\STATE Solve \textbf{DSP} $\rightarrow$ obtain $M^{(k)}$.
		\STATE Solve \textbf{Static RLP} on $M^{(k)}$ $\rightarrow$ get $\tilde{x}^{(k)}$.
		\STATE Solve \textbf{Adversarial Problem:}\\
		\[\max_{\mathcal{U}_v,\mathcal{U}_{e}} \ \sum_v \tilde{c}_v \tilde{x}_v^{(k)}\]
		\STATE Update $(\tilde{c}^{(k+1)}, c^{(k+1)})$.
		\STATE $k \leftarrow k + 1$.
		\UNTIL{$|Z^{(k+1)} - Z^{(k)}| / Z^{(k)} \le \varepsilon$}
		\STATE \textbf{Output:} Robust regenerator placement $\tilde{x}^{(k)}$.
	\end{algorithmic}
\end{algorithm}


\section{Learning Based Hide-and-Seek Game}
\label{sec:learning-based-hide-and-seek-game}

The robust weighted RLP under both dynamic and static budgeted uncertainty set can be reformulated as a \textit{learning-based hide-and-seek game} (HSL) between a \textit{seeker} called the network planner and a \textit{hider} referred to as the adversary. The seeker's goal is to minimize total installation cost while maintaining network connectivity at all time, whereas the adversary aims to maximize cost or disrupt connectivity by revealing adverse realizations of uncertain parameters.


\subsection{Game Definition}
\label{subsec:game-definition}

Let $\mathcal{T}$ denote a discrete set of time periods. The seeker chooses a regenerator placement vector $\tilde{\pmb{x}}\in\{0,1\}^n$, while the adversary selects the nodes and edges deviating from their nominal cost and length by setting $(\pmb{y}^t, \pmb{z})$ under the following fixed budgets.

\footnotesize
\begin{align*}
& \sum_{e\in E} y_e^t \le \Gamma_e & y_e^t\in \{0,1\}\\
& \sum_{v\in V} z_v \le \Gamma_v & z_v \in \{0,1\}
\end{align*}
\normalsize

The uncertain parameters are defined as $c_e^t = \underline{c}_e + d_e^t y_e^t$ and $\tilde{c}_v = \underline{\tilde{c}}_v + \tilde{d}_v z_v$ where $d_e^t$ are the dynamically adjusted deviations. In addition, the loss function is:

\footnotesize
\[
L^t(\tilde{x},\tilde{c},c^t) = \sum_{v\in V} \tilde{c}_v \tilde{x}_v,
\]
\normalsize
subject to network connectivity constraints based on $c^t$. Thus, the corresponding robust formulation is:

\footnotesize
\begin{align*}
\min_{\tilde{x}\in\mathcal{X}} \ \max_{t\in\mathcal{T}} \ & \ L^t(\tilde{x},\tilde{c},c^t),\\
\text{s.t.} \ & \ (y^t,z)\in(\mathcal{Y}^t,\mathcal{Z})
\end{align*}
\normalsize
where $\mathcal{Y}^t$ and $\mathcal{Z}$ are the feasible adversarial sets.

\subsection{Learning Mechanism}

To model the adaptive interaction between seeker and adversary, we employ a \textit{gradient-based} learning rule. The adversary updates the edge deviation levels $d_e^t$ in the direction of increasing seeker's cost, while the seeker adapts its placement decisions $\tilde{\pmb{x}}$ to minimize the expected cost. In this sense, at iteration $k$, for each edge $e$ and time $t$, the deviation parameter is updated in the following way:

\footnotesize
\[
d_e^{t,(k+1)} = \Pi_{[0,d_e]}\Big(d_e^{t,(k)} 
+ \eta_d \frac{\partial L^t(\tilde{x}^{(k)},\tilde{c},c^{t,(k)})}{\partial d_e^t}\Big),
\]
\normalsize
where $\eta_d>0$ is the adversary learning rate, and $\Pi_{[0,d_e]}(\cdot)$ projects the value back into the feasible range. The seeker reacts to the observed deviations by solving the following problem to reduce the expected cost under the updated uncertainty realization.

\footnotesize
\[
\tilde{\pmb{x}}^{(k+1)} = 
\arg\min_{\tilde{x}\in\mathcal{X}} 
\Big\{\mathbb{E}_{t\in\mathcal{T}}[L^t(\tilde{x},\tilde{c},c^{t,(k+1)})]\Big\},
\]
\normalsize
This gradient-based process, which is summarized in Algorithm~\ref{alg:hsl}, drives the system toward a saddle-point equilibrium, where neither players can unilaterally improve their objective.

\begin{algorithm}[H]
	\caption{Learning-Based Hide-and-Seek Algorithm}
	\label{alg:hsl}
	\begin{algorithmic}[1]
		\STATE \textbf{Input:} Initial learning rate $\eta_d$, and tolerance $\varepsilon$.
		\STATE Initialize $k \leftarrow 0$, $d_e^{t,(0)} \leftarrow d_e$, and $\tilde{x}^{(0)} \leftarrow \emptyset$.
		\REPEAT
		\STATE \textbf{Seeker step:} 
		Solve the following to obtain $\tilde{x}^{(k)}$
		\[\min_{\tilde{x}\in\mathcal{X}} \ \mathbb{E}_{t\in\mathcal{T}}[L^t(\tilde{x},\tilde{c},c^{t,(k)})]\].
		\STATE \textbf{Hider step:} 
		For each $(e,t)$, update deviations:
		\[
		d_e^{t,(k+1)} = \Pi_{[0,d_e]}\Big(d_e^{t,(k)} + \eta_d \frac{\partial L^t(\tilde{x}^{(k)},\tilde{c},c^{t,(k)})}{\partial d_e^t}\Big).
		\]
		\STATE Update edge length $c_e^{t,(k+1)} = \underline{c}_e + d_e^{t,(k+1)} y_e^t$.
		\STATE $k \leftarrow k + 1$.
		\UNTIL{$|L^t(\tilde{x}^{(k+1)},\tilde{c},c^{t,(k+1)}) - L^t(\tilde{x}^{(k)},\tilde{c},c^{t,(k)})| < \varepsilon$}
		\STATE \textbf{Output:} Placement $\tilde{x}^* = \tilde{x}^{(k)}$, stabilized deviations $d_e^{t,*}$.
	\end{algorithmic}
\end{algorithm}

This learning rule simulates a feedback-driven adaptation process. In other words, the adversary intensifies deviations where the seeker is most vulnerable, while the seeker reallocates regenerators to minimize the expected total cost.

\subsection{Numerical Example}
\label{subsec:numerical-example-gradient-hsl}

In this section, we demonstrate the HSL algorithm on the following five-node optical network.

\footnotesize
\[V = \{1,2,3,4,5\}, \quad E = \{(1,2),(2,3),(3,4),(4,5),(2,5)\}\]
\normalsize
The nominal and deviation parameters are shown in Table~\ref{tab:gradient-hsl-example}. We also set $\Gamma_e=2$, $\Gamma_v=1$, $\eta_d=0.15$, $d_{\max}=10$, and $\mathcal{T}=\{1,2,3\}$. Initially, $d_e^{t,(0)} = d_e$ for all $e,t$.

\begin{table}[H]
	\centering
	\caption{Nominal parameters for HSL example.}
	\begin{tabular}{ccc}
		\toprule
		Edge & $\underline{c}_{ij}$ & $d_{ij}$ \\
		\midrule
		(1,2) & 4 & 1.0 \\
		(2,3) & 3 & 0.8 \\
		(3,4) & 5 & 1.5 \\
		(4,5) & 4 & 1.0 \\
		(2,5) & 6 & 1.2 \\
		\bottomrule
	\end{tabular}
	\hspace{1cm}
	\begin{tabular}{ccc}
		\toprule
		Node & $\underline{\tilde{c}}_v$ & $\tilde{d}_v$ \\
		\midrule
		1 & 8 & 2 \\
		2 & 10 & 3 \\
		3 & 9 & 2 \\
		4 & 7 & 1 \\
		5 & 11 & 2 \\
		\bottomrule
	\end{tabular}
	\label{tab:gradient-hsl-example}
\end{table}

\textit{Iteration 1 (Nominal initialization):} The seeker selects nodes $\{2,4\}$ for regeneration, achieving a total cost of $17$. The adversary computes $\partial L^t / \partial d_e^t$ and finds that edges $(3,4)$ and $(2,5)$ contribute most to cost sensitivity.

\textit{Iteration 2 (First gradient update):} The adversary increases deviations on $(3,4)$ and $(2,5)$, thus we have:

\footnotesize
\begin{align*}
& d_{(3,4)}^{t,(1)} = 1.5 + 0.15 \times (0.8)=1.62\\
& d_{(2,5)}^{t,(1)} = 1.2 + 0.15(0.5)=1.27
\end{align*}
\normalsize
therefore, the seeker updates placement to $\{2,3,4\}$ with total cost $21.3$.

\textit{Iteration 3 (Learning convergence)}: Next update marginally increase deviations near $(3,4)$, but $\tilde{\pmb{x}}$ stabilizes at $\{2,4,5\}$, achieving robust connectivity with mean cost $20.6$ across all $t\in\mathcal{T}$.

The HSL algorithm converges in three iterations, maintaining full connectivity while achieving approximately 10\% lower cost than the deterministic worst-case solution. The gradient-based adaptation allows the adversary to learn vulnerability patterns efficiently without explicit enumeration of all uncertainty realizations.


\section{Experimental Results}\label{sec:experiments}

In this section, we investigate the performance of our proposed models to solve the robust weighted RLP under both static and dynamic budgeted uncertainty sets. To illustrate the effectiveness of our models, we make a comparison with the results obtained by both the DWC and RSB cases. In fact, DWC and RSB can be considered as the state-of-the-art methods to address deterministic and uncertain RLP, respectively.

\subsection{Experimental Setup}

We conduct four sets of numerical experiments to evaluate and compare the performance of the proposed solution methods. In each experiment, we vary key parameters, including the number of nodes $n$ and the adversary’s budget for attacking edges and nodes, denoted by $\Gamma_e$ and $\Gamma_v$, respectively. Throughout all experiments, we fix the maximum possible transmission distance to $d_{\max} = 1000$ and set the network density to $dens = 0.3$.

All test instances are generated using a unified procedure. For each edge $e \in E$, the nominal length $\underline{c}_e$ is drawn uniformly at random from the set $\{350,351,\ldots,600\}$, and the maximum deviation $d_e$ is selected from $\{1,2,\ldots,250\}$. Similarly, for each node $v \in V$, the nominal installation cost $\underline{\tilde{c}}_v$ is randomly chosen from $\{250,251,\ldots,300\}$, and the corresponding maximum deviation $\tilde{d}_v$ is drawn from $\{1,2,\ldots,50\}$.

In the first experiment (Exp-1), we vary the network size while keeping all other parameters fixed. Specifically, we consider $n = 10,12,\ldots,30$ and set $\Gamma_e = 2$ and $\Gamma_v = 2$. In the second experiment (Exp-2), we fix the network size at $n = 25$ and vary the uncertainty budgets by setting $\Gamma_e \in \{1,2\}$ and $\Gamma_v \in \{1,2,3\}$. All instances in Exp-1 and Exp-2 are solved using the DWC, RSB, and RDB approaches, allowing for a direct comparison of their performance.

To evaluate scalability, we generate larger instances following the same procedure. In the third experiment (Exp-3), we set $n = 40,42,\ldots,60$ and fix $\Gamma_e = 2$ and $\Gamma_v = 2$. In the fourth experiment (Exp-4), we fix the network size at $n = 50$ and consider $\Gamma_e \in \{1,2\}$ and $\Gamma_v \in \{1,2,3\}$, while additionally setting $\eta_d = 0.1$ and limiting the maximum number of iterations to 10. All instances in Exp-3 and Exp-4 are solved using the DWC, BDC, CCG, IRO, and HSL methods for comparative analysis.

For each experiment and every parameter combination, we generate and solve 50 independent instances using IBM ILOG CPLEX version 22.1.1 on an Intel\textsuperscript{\textregistered} Core\textsuperscript{TM} i7-10510U CPU running at 1.80~GHz with 15~GB of RAM. All computations are performed using a single processing thread.

\subsection{Results of Small Instances}

The average cost-related results for Exp-1 are summarized in Table~\ref{table:exp1-results}. The column labeled \emph{R-DWC} reports the average percentage improvement in the final objective value achieved by each method relative to the DWC instance. Similarly, the column \emph{R-RSB} indicates the average percentage improvement with respect to the RSB approach. The results show that the performance gap between RSB and DWC is relatively small, indicating only marginal benefits from incorporating static robustness alone. In contrast, the proposed RDB approach consistently outperforms both RSB and DWC, achieving cost reductions of up to approximately 10\%, which highlights the effectiveness of incorporating dynamic uncertainty into the modeling framework.

Figure~\ref{fig:exp1-results} presents both the average and instance-wise comparisons (using \textit{Performance Profile}, see Appendix~\ref{app:performance-profile}) of the computational time for Exp-1. As expected, the solution time for all three approaches increases with the number of nodes, reflecting the growing complexity of the instances. While the RDB method yields superior solution quality, it also incurs the highest computational cost, making it the slowest among the considered approaches in this experiment.

\begin{table}[htbp]
	\begin{center}
		\caption{Results of Exp-1}\label{table:exp1-results}
		\begin{tabular}{|c|c|c|c|c|c|c|}
			\hline
			\multirow{2}{*}{n} & \multicolumn{1}{|c|}{DWC} & \multicolumn{2}{|c|}{RSB} & \multicolumn{3}{|c|}{RDB} \\
			\cline{2-7}
			 & Cost & Cost & R-DWC & Cost & R-DWC & R-RSB \\
			\hline
			10 & 855.72 & 846.64 & 1.02 $\%$ & 818.38 & 04.27 $\%$ & 03.29 $\%$ \\
			12 & 875.04 & 865.04 & 1.10 $\%$ & 835.58 & 04.37 $\%$ & 03.33 $\%$ \\
			14 & 858.68 & 850.18 & 0.92 $\%$ & 801.28 & 06.48 $\%$ & 05.62 $\%$ \\
			16 & 877.56 & 868.58 & 0.98 $\%$ & 824.60 & 05.81 $\%$ & 04.88 $\%$ \\
			18 & 905.48 & 892.26 & 1.40 $\%$ & 856.44 & 05.17 $\%$ & 03.84 $\%$ \\
			20 & 858.66 & 849.48 & 1.02 $\%$ & 797.26 & 06.70 $\%$ & 05.73 $\%$ \\
			22 & 881.54 & 873.10 & 0.92 $\%$ & 777.54 & 11.24 $\%$ & 10.44 $\%$ \\
			24 & 879.12 & 869.36 & 1.08 $\%$ & 784.56 & 10.50 $\%$ & 09.52 $\%$ \\
			26 & 871.98 & 862.50 & 1.04 $\%$ & 784.60 & 09.51 $\%$ & 08.56 $\%$ \\
			28 & 872.80 & 863.54 & 1.03 $\%$ & 770.44 & 11.19 $\%$ & 10.29 $\%$ \\
			30 & 866.70 & 858.80 & 0.86 $\%$ & 762.74 & 11.69 $\%$ & 10.92 $\%$ \\
			\hline
		\end{tabular}
	\end{center}
\end{table}
\begin{figure}[htbp]
	\begin{center}
		\subfigure[Solution Time]{\includegraphics[width=0.23\textwidth]{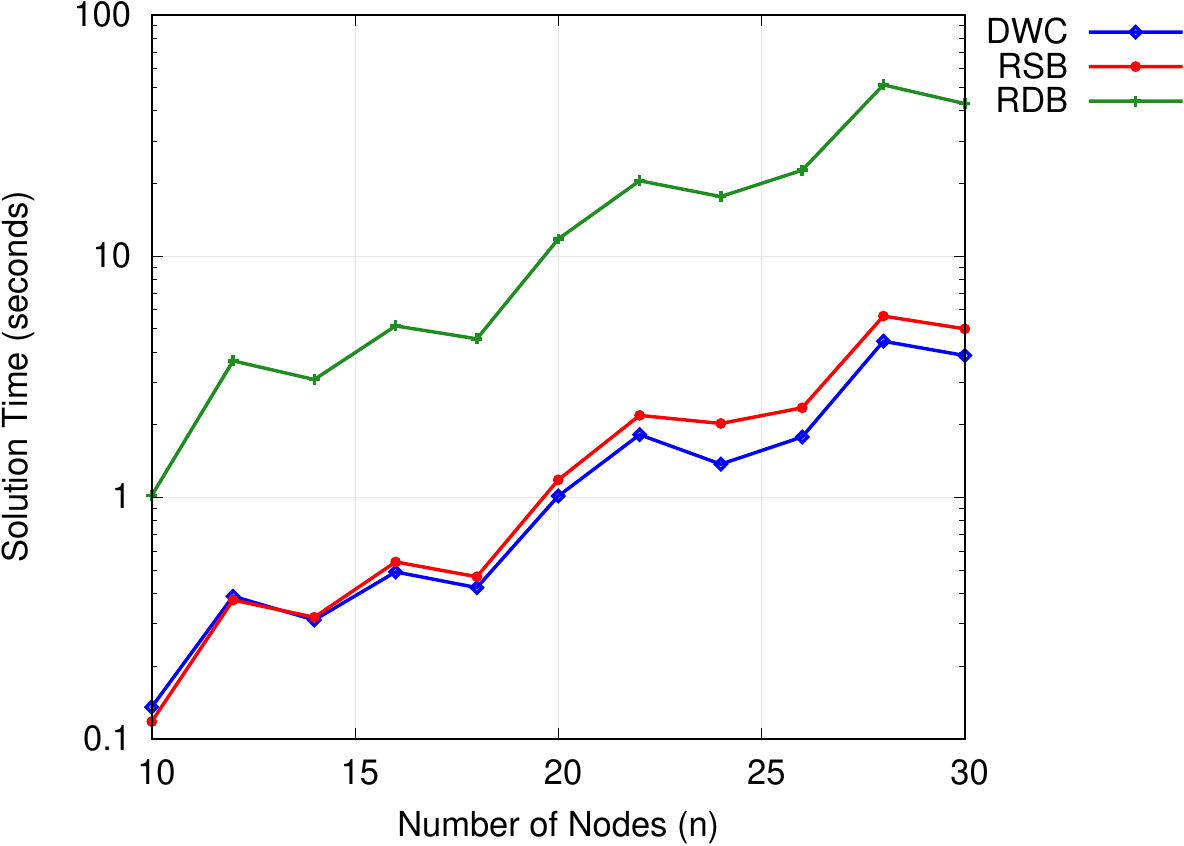}}
		\subfigure[Performance Profile]{\includegraphics[width=0.23\textwidth]{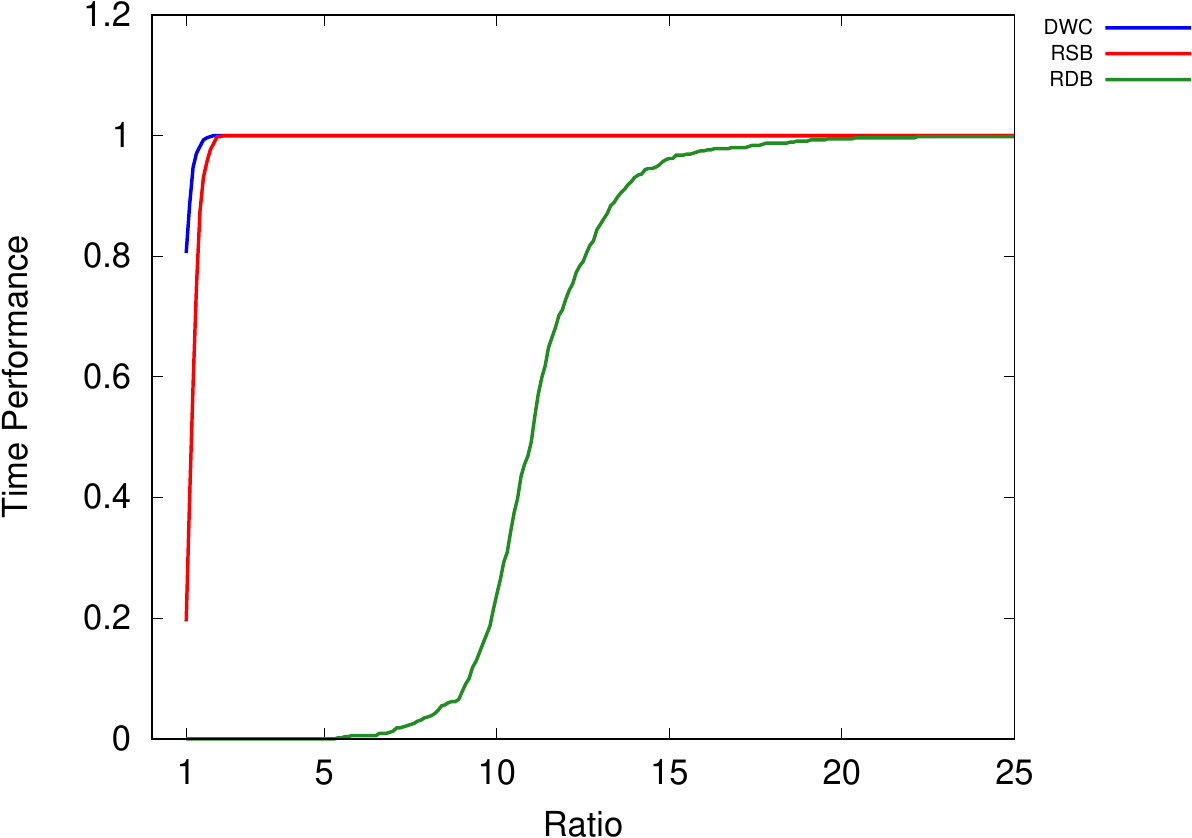}}
	\end{center}
	\caption{Time Performance of Exp-1}\label{fig:exp1-results}
\end{figure}

Table~\ref{table:exp2-results} reports the average cost-related outcomes for Exp-2. The results indicate that, particularly when $\Gamma_e = 2$, the difference in performance between RSB and DWC remains limited, suggesting that the added conservatism of static robustness leads to only minor improvements in solution quality. By contrast, the RDB approach exhibits a clear and consistent advantage over both DWC and RSB. In particular, it delivers substantial cost savings of up to about 28\% when $\Gamma_e = 1$ and around 11\% when $\Gamma_e = 2$, underscoring the benefit of explicitly accounting for dynamic uncertainty in edge parameters.

The computational time results for Exp-2, illustrated in Figure~\ref{fig:exp2-results}, show that variations in the uncertainty budgets $\Gamma_e$ and $\Gamma_v$ have a negligible impact on the overall solution time of each method. The relative ranking of the approaches remains unchanged: RDB requires the largest computational effort, while DWC and RSB are comparatively faster. These observations are consistent with those of Exp-1.

\begin{table}[htbp]
	\begin{center}
		\caption{Results of Exp-2}\label{table:exp2-results}
		\begin{tabular}{|c|c|c|c|c|c|c|c|}
			\hline
			\multirow{2}{*}{$\Gamma_e$} & \multirow{2}{*}{$\Gamma_v$} & \multicolumn{1}{|c|}{DWC} & \multicolumn{2}{|c|}{RSB} & \multicolumn{3}{|c|}{RDB} \\
			\cline{3-8}
			& & Cost & Cost & R-DWC & Cost & R-DWC & R-RSB \\
			\hline
			\multirow{3}{*}{1}  & 1 & 859.5 & 721.74 & 15.59 & 612.50 & 28.03 & 13.76 \\
								& 2 & 859.5 & 739.10 & 13.60 & 625.56 & 26.53 & 13.99 \\
								& 3 & 859.5 & 744.60 & 12.98 & 626.98 & 26.38 & 14.34 \\
			\hline
			\multirow{3}{*}{2} 	& 1 & 859.5 & 830.00 & 03.34 & 761.54 & 10.95 & 07.82 \\
								& 2 & 859.5 & 850.46 & 01.00 & 779.80 & 08.85 & 07.94 \\
								& 3 & 859.5 & 858.98 & 00.05 & 785.50 & 08.21 & 08.17 \\
			\hline
		\end{tabular}
	\end{center}
\end{table}
\begin{figure}[htbp]
	\begin{center}
		\subfigure[Solution Time]{\includegraphics[width=0.24\textwidth]{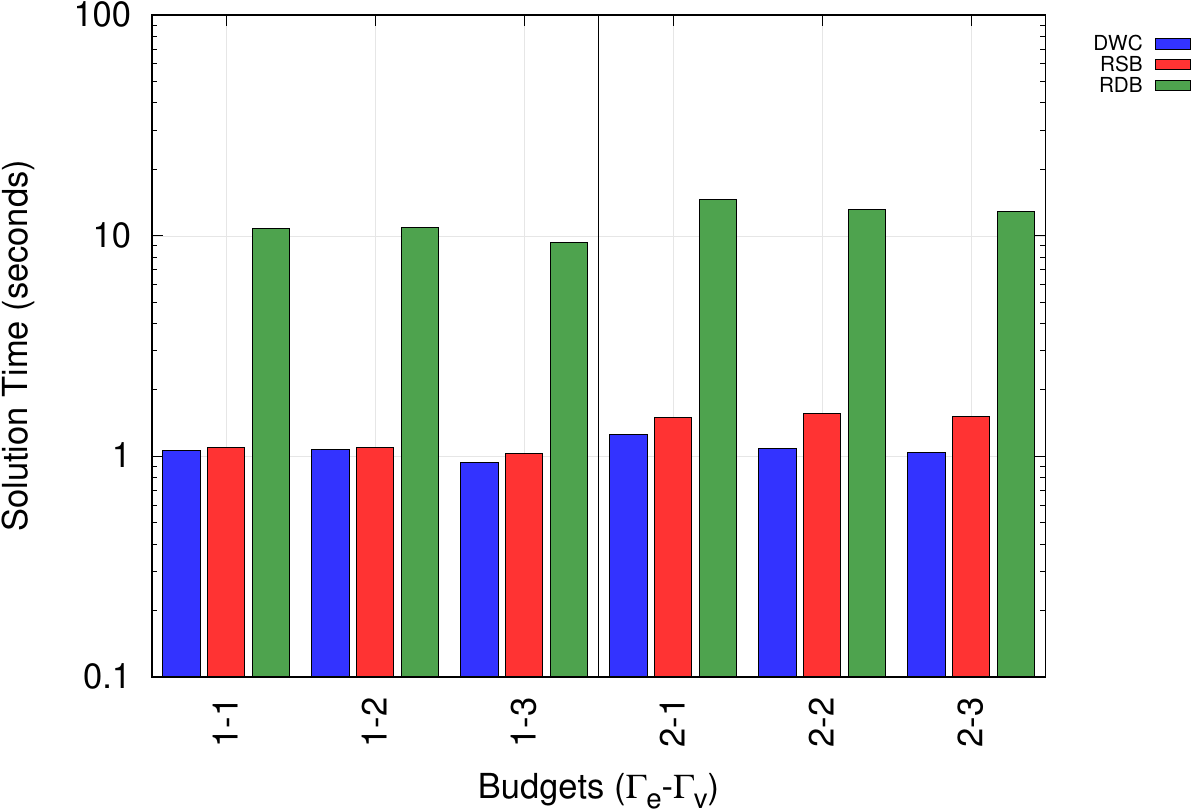}}
		\subfigure[Performance Profile]{\includegraphics[width=0.23\textwidth]{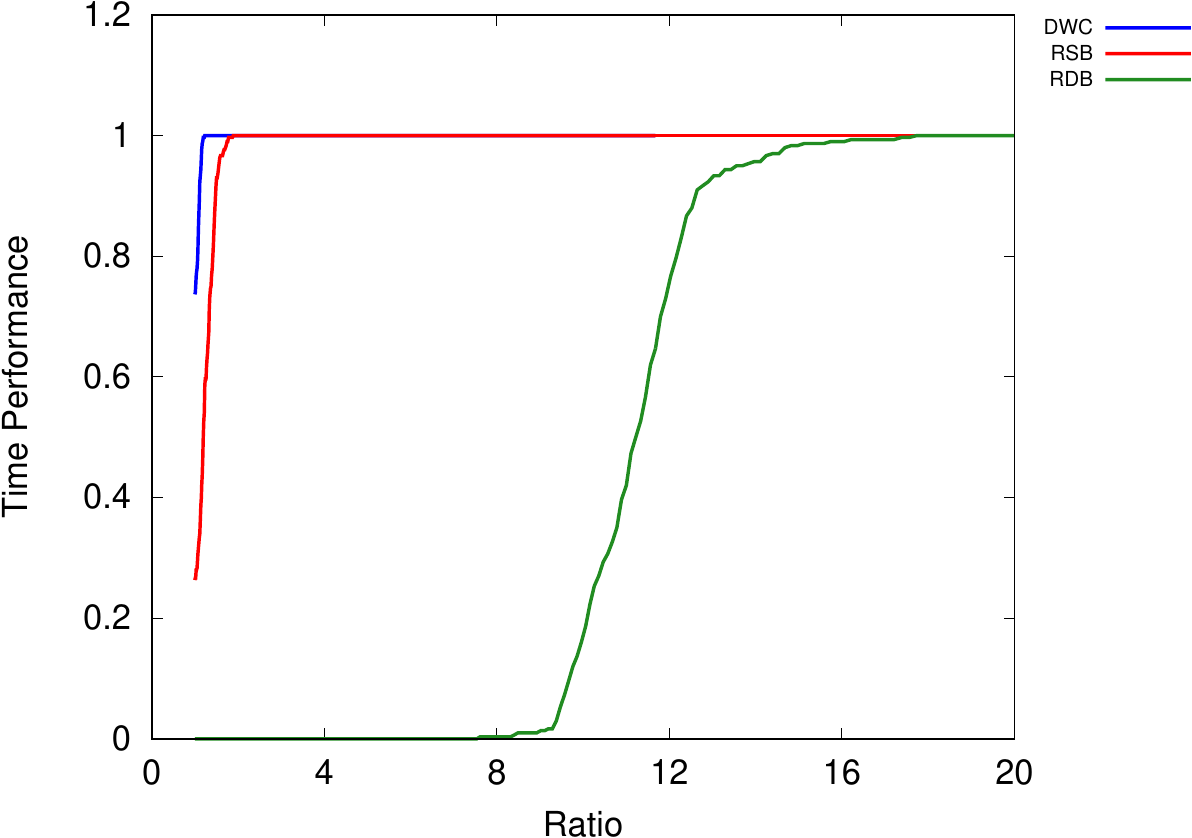}}
	\end{center}
	\caption{Time Performance of Exp-2}\label{fig:exp2-results}
\end{figure}

\subsection{Results of Large Instances}

The average cost-related results for Exp-3 are reported in Table~\ref{table:exp3-results}. The column labeled \emph{RDB Cost} corresponds to the objective value achieved by all alternative methods, since each of them is solved to optimality in this experiment. The results show that the proposed solution approaches are able to further reduce the final cost by up to 2.40\%, demonstrating that even for larger instances with tighter computational limits, meaningful improvements can still be obtained. In addition, the table highlights clear differences in convergence behavior across methods: the CCG approach requires the smallest number of iterations to reach optimality, followed by the BDC method. Although both HSL and IRO are capable of achieving comparable cost reductions, they typically require largerer number of iterations, which reflects a slower convergence rate.

Figure~\ref{fig:exp3-results} illustrates the average and instance-wise computational times for Exp-3. Consistent with the iteration counts, CCG exhibits the lowest solution times among all considered methods, in most cases even outperforming the DWC method. This result positions CCG as the most efficient approach for solving large-scale instances of the problem. Moreover, the observed computational times closely follow the relative iteration rankings of the methods, further confirming the strong relationship between convergence speed and overall computational performance.

\begin{table}[htbp]
	\begin{center}
		\caption{Results of Exp-3}\label{table:exp3-results}
		\begin{tabular}{|c|c|c|c|c|c|c|c|}
			\hline
			\multirow{2}{*}{$n$} & DWC & RDB & \multirow{2}{*}{R-DWC} & \multicolumn{4}{|c|}{Iterations}\\
			\cline{5-8}
			& Cost & Cost &  & BDC & CCG & IRO & HSL \\
			\hline
			40 & 864.12 & 855.04 & 0.97 & 2.00 & 2.00 & 2.40 & 2.06 \\
			42 & 876.28 & 866.56 & 1.06 & 2.00 & 2.00 & 2.48 & 2.18 \\
			44 & 842.52 & 834.12 & 0.97 & 2.04 & 2.00 & 2.56 & 2.22 \\
			46 & 847.44 & 841.80 & 0.64 & 2.06 & 2.02 & 2.58 & 2.26 \\
			48 & 849.60 & 828.68 & 2.18 & 2.06 & 2.02 & 2.60 & 2.34 \\
			50 & 837.32 & 829.56 & 0.86 & 2.12 & 2.04 & 2.60 & 2.36 \\
			52 & 844.00 & 834.48 & 1.10 & 2.12 & 2.06 & 2.62 & 2.36 \\
			54 & 808.12 & 789.76 & 2.05 & 2.16 & 2.08 & 2.68 & 2.38 \\
			56 & 807.04 & 791.76 & 1.82 & 2.18 & 2.12 & 2.68 & 2.42 \\
			58 & 825.80 & 805.72 & 2.40 & 2.18 & 2.12 & 2.78 & 2.42 \\
			60 & 794.48 & 788.84 & 0.68 & 2.20 & 2.16 & 2.86 & 2.44 \\
			\hline
		\end{tabular}
	\end{center}
\end{table}

\begin{figure}[htbp]
	\begin{center}
		\subfigure[Solution Time]{\includegraphics[width=0.23\textwidth]{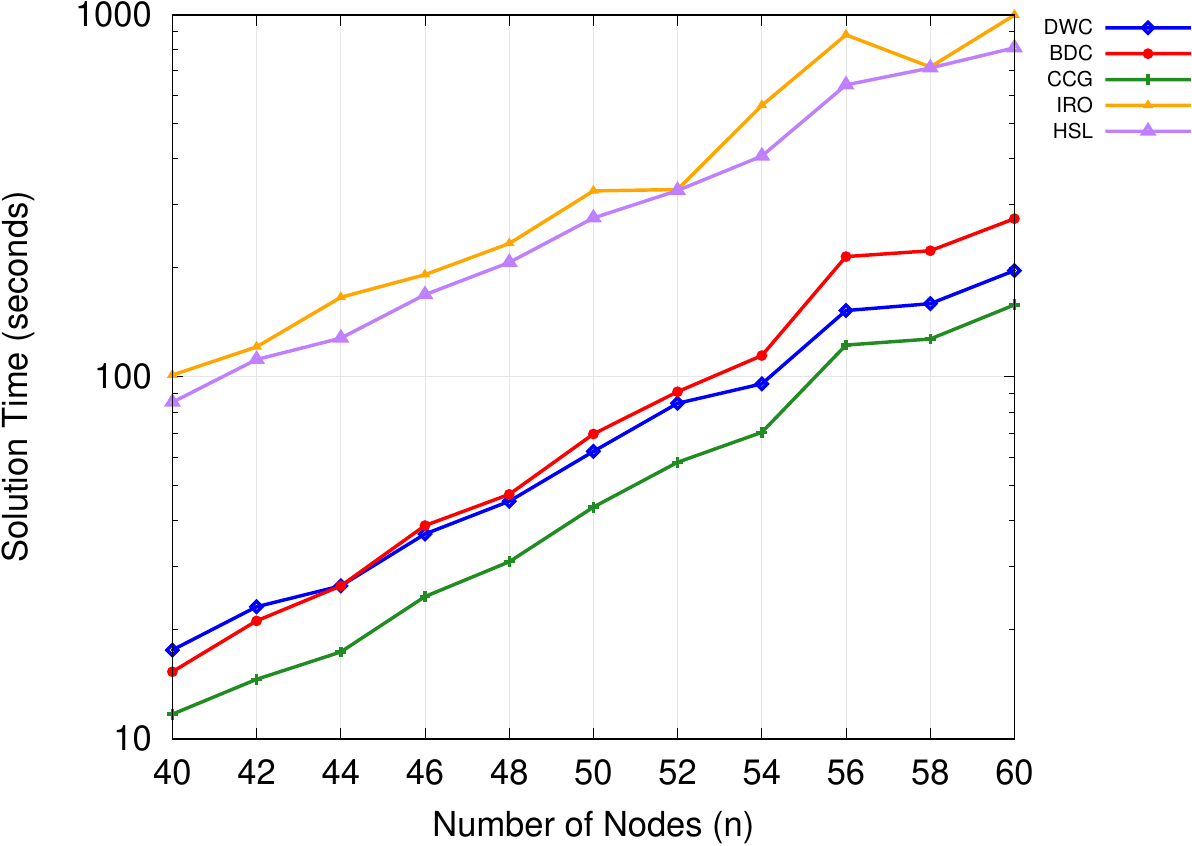}}
		\subfigure[Performance Profile]{\includegraphics[width=0.23\textwidth]{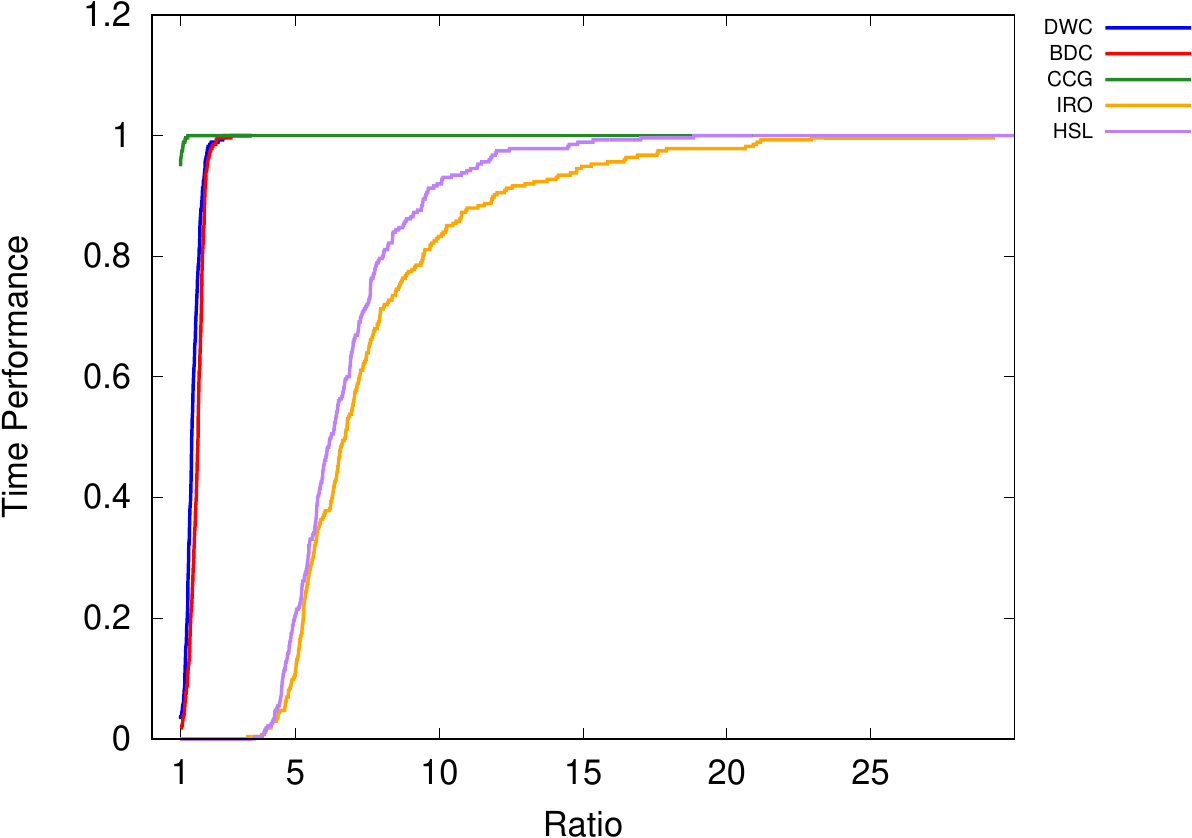}}
	\end{center}
	\caption{Time Performance of Exp-3}\label{fig:exp3-results}
\end{figure}

Table~\ref{table:exp4-results} presents the average cost-related outcomes for Exp-4. The results show that, especially in the case where $\Gamma_e = 1$, the performance gap between the proposed solution approaches and the DWC is substantial, reaching reductions of up to 31\% in the final cost. This observation underscores the benefit of incorporating more advanced robust optimization strategies when the adversarial budget on edges is limited. Furthermore, the table reveals notable differences in convergence behavior among the methods, which are consistent with the trends observed in Exp-3, with decomposition-based approaches converging in fewer iterations compared to iterative methods.

Figure~\ref{fig:exp4-results} depicts both the average and instance-wise computational times for Exp-4. In line with the iteration statistics, the CCG method consistently achieves the shortest solution times among all approaches and, in many instances, even outperforms the DWC model. These results identify CCG as the most computationally efficient method for handling large-scale instances of the problem. Additionally, the close correspondence between solution times and iteration counts further confirms the strong link between convergence speed and overall computational performance.

\begin{table}[htbp]
	\begin{center}
		\caption{Results of Exp-4}\label{table:exp4-results}
		\begin{tabular}{|c|c|c|c|c|c|c|c|c|}
			\hline
			\multirow{2}{*}{$\Gamma_e$} & \multirow{2}{*}{$\Gamma_v$} & DWC & RDB & \multirow{2}{*}{RDWC} & \multicolumn{4}{|c|}{Iterations}\\
			\cline{6-9}
			 & & Cost & Cost &  & BDC & CCG & IRO & HSL \\
			\hline
			\multirow{3}{*}{1}  & 1 & 837.3 & 569.5 & 31.25 & 2.00 & 2.00 & 2.16 & 2.12 \\
								& 2 & 837.3 & 578.4 & 30.21 & 2.02 & 2.00 & 2.16 & 2.12 \\
								& 3 & 837.3 & 578.8 & 30.16 & 2.04 & 2.00 & 2.16 & 2.14 \\
			\hline
			\multirow{3}{*}{2}  & 1 & 837.3 & 812.2 & 02.92 & 2.04 & 2.02 & 2.44 & 2.36 \\
								& 2 & 837.3 & 829.5 & 00.86 & 2.06 & 2.04 & 2.46 & 2.36 \\
								& 3 & 837.3 & 835.1 & 00.22 & 2.10 & 2.06 & 2.48 & 2.40 \\
			\hline
		\end{tabular}
	\end{center}
\end{table}

\begin{figure}[htbp]
	\begin{center}
		\subfigure[Solution Time]{\includegraphics[width=0.24\textwidth]{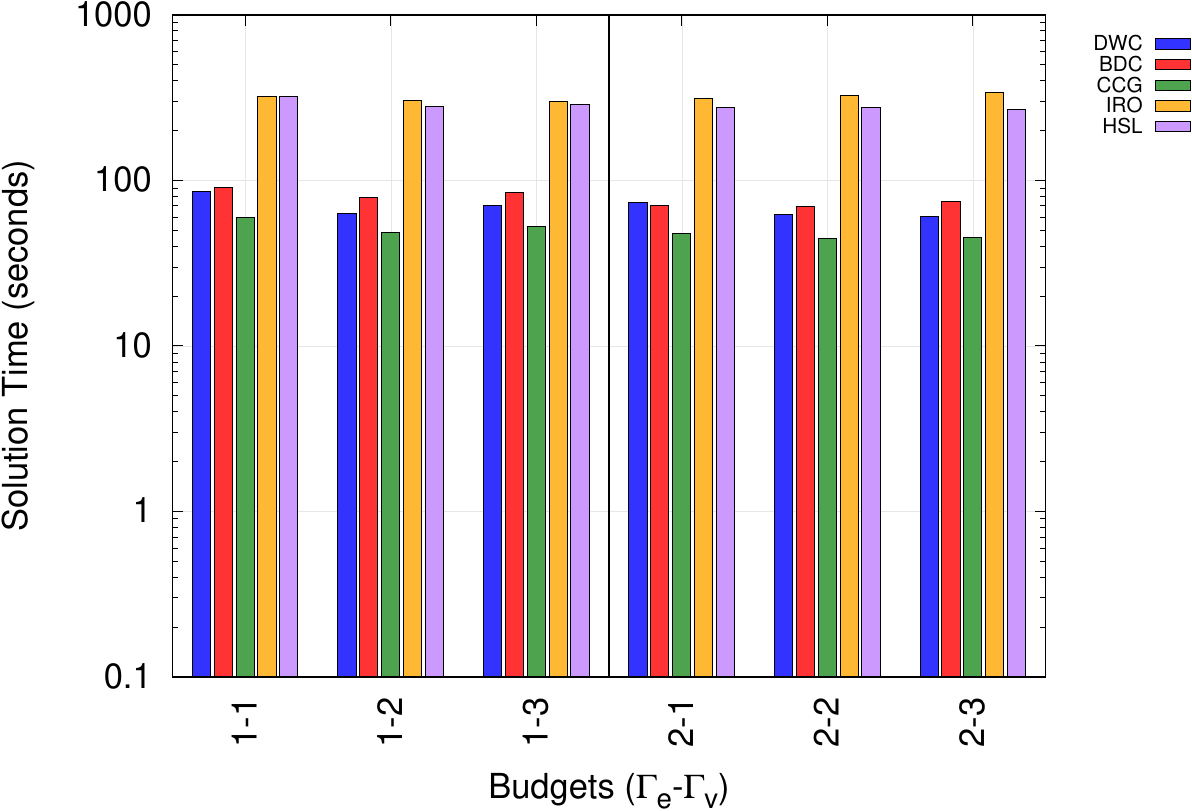}}
		\subfigure[Performance Profile]{\includegraphics[width=0.23\textwidth]{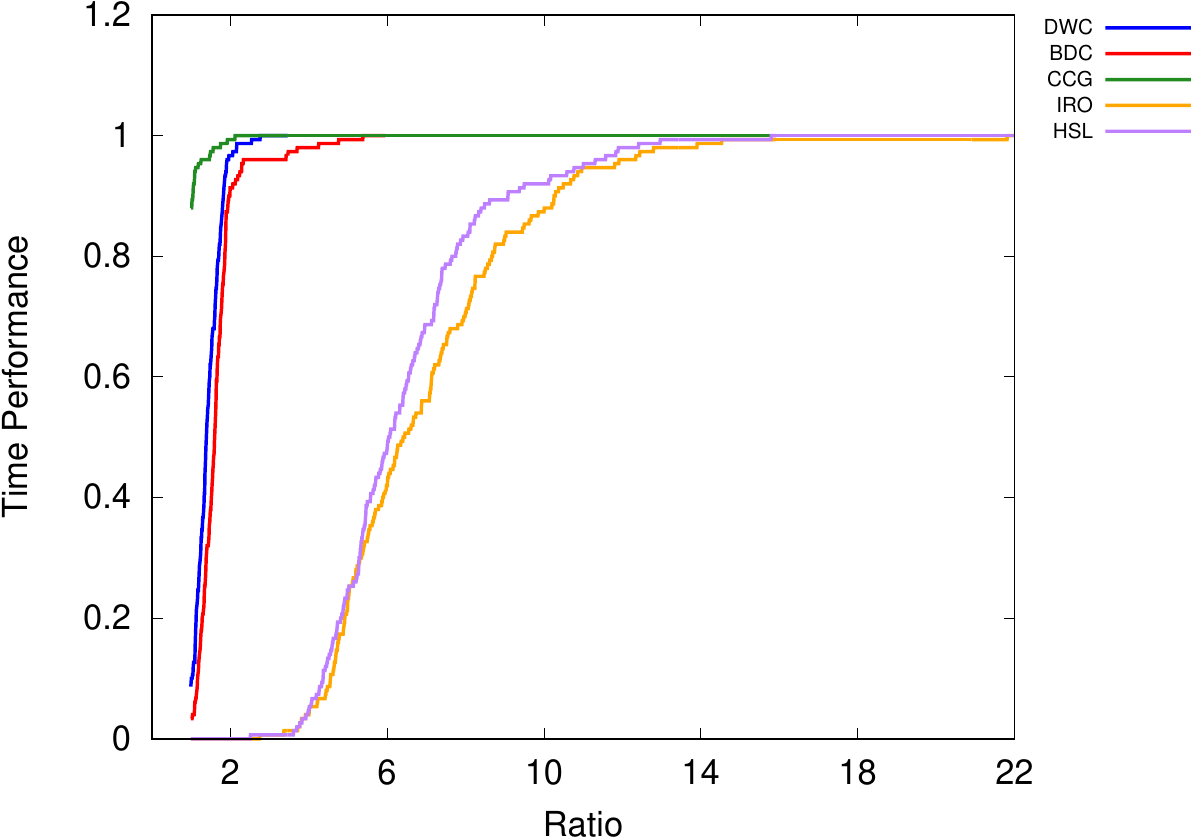}}
	\end{center}
	\caption{Time Performance of Exp-4}\label{fig:exp4-results}
\end{figure}


\section{Conclusion}\label{sec:conclusion}

The rapid evolution of communication technologies has fundamentally reshaped the way organizations and enterprises operate and collaborate. The widespread adoption of virtualization, edge computing, and intelligent automation has intensified the demand for integrated, adaptive, and resilient network infrastructures. In this context, robustness and stability have emerged as critical design objectives, not only to preserve network performance under uncertainty but also to ensure the continuity of essential services in the presence of disruptions or adverse operating conditions.

In this paper, we introduced a novel survivable network design framework in which the network is subject to two distinct sources of uncertainty. Node installation costs are modeled through a static budgeted uncertainty set, while edge lengths are captured via a dynamic budgeted uncertainty set that evolves over time and directly affects the network structure. To address this setting, we developed three modeling paradigms. First, we examined a deterministic worst-case formulation in which all node and edge parameters simultaneously deviate to their upper bounds. Next, we formulated a robust model in which both node costs and edge lengths are governed by static budgeted uncertainty. Finally, we proposed a new and more realistic formulation that combines static budgeted uncertainty for node costs with dynamic budgeted uncertainty for edge lengths, thereby capturing temporal variations in network conditions.

To effectively solve the resulting models, we developed several alternative solution approaches, including Benders decomposition, column-and-constraint generation, and iterative robust optimization. In addition, we proposed a learning-based hide-and-seek framework that interprets the robust optimization process as a repeated game between a decision-maker and an adversary. Extensive computational experiments demonstrated the effectiveness of the proposed methods, particularly for large-scale instances, and highlighted the advantages of incorporating dynamic uncertainty into the modeling framework.

Several promising directions for future research remain open. Developing efficient heuristic or metaheuristic methods for the graph transformation and regenerator placement components, either separately or in an integrated manner, could further improve scalability. Moreover, the framework could be extended by considering alternative uncertainty representations, such as dynamic interval uncertainty, or by explicitly modeling edge failures and survivability constraints. These extensions would further enhance the practical relevance of robust regenerator placement models in next-generation communication networks.



\appendices

\section{Proof of Theorem~\ref{theorem:complexity}}\label{app:proof-theorem-complexity}
\begin{proof}
	In our setting, we consider two optimization problems under budgeted uncertainty: the shortest path problem and the regenerator location problem.
	
	Inspired by \cite{bertsimas2003robust}, it is known that robust combinatorial optimization problems with budgeted uncertainty sets can be decomposed into $O(n)$ nominal problems, and therefore the computational complexity of the robust formulation coincides with that of its nominal counterpart. In addition, based on the results of \cite{ahuja1994network}, the deterministic shortest path problem is solvable in polynomial time. As a result, the robust counterpart of the shortest path problem under budgeted uncertainty admits a polynomial-time solution.
	
	On the other hand, due to the proof provided in \cite{li2017regenerator}, the RLP is NP-hard. Consequently, the robust weighted RLP under a budgeted uncertainty set remains NP-hard. It follows that the problem considered in this paper is NP-hard.
\end{proof}

\section{Performance Profile}\label{app:performance-profile}

To evaluate solution times for each problem, and in addition to presenting the normal solution times, we use performance profiles as introduced in \cite{goerigk2024robust}. We briefly recall this concept: Let $\cS$ be the set of considered models, $\K$ the set of instances and $t_{k,s}$ the run time of model $s$ on instance $k$. We assume $t_{k,s}$ is set to infinity (or large enough) if model $s$ does not solve instance $k$ within the time limit. The percentage of instances for which the performance ratio of solver $s$ is within a factor $\tau \geq 1$ of the best ratio of all solvers is given by:

\footnotesize
\[k_s (\tau) = \frac{1}{\lvert \K \rvert} \; \Bigg\lvert \Bigg\{ k\in \K \; \vert \; \frac{t_{k,s}}{\text{min}_{\hat{s}\in \cS} t_{k,\hat{s}}}  \leq \tau \Bigg\} \Bigg \lvert  \]
\normalsize
Hence, the function $k_s$ can be viewed as the distribution function for the performance ratio, which is plotted in a performance profile for each model. It must be noticed that $k_s (\tau) \in [0,1]$, and better models have higher $k_s (\tau)$.

\end{document}